\documentclass[conference]{ieeeconf}
\IEEEoverridecommandlockouts
\usepackage{style}
\usepackage{tikz}

\usepackage[final]{changes} % 最終提出時に切り替える用
\def\BibTeX{{\rm B\kern-.05em{\sc i\kern-.025em b}\kern-.08em
    T\kern-.1667em\lower.7ex\hbox{E}\kern-.125emX}}
\begin{document}
%%%%%%%%%%%%%%%%%%%%%%%%%%%%%%%%%%%%%%%%%%%%%%%%
%%%%%%%%%%%%%%%%%%%%%%%%%%%%%%%%%%%%%%%%%%%%%%%%
%%%%%%%%%%%%%%%%%%%%%%%%%%%%%%%%%%%%%%%%%%%%%%%%
% title and authors
%%%%%%%%%%%%%%%%%%%%%%%%%%%%%%%%%%%%%%%%%%%%%%%%
%%%%%%%%%%%%%%%%%%%%%%%%%%%%%%%%%%%%%%%%%%%%%%%%
%%%%%%%%%%%%%%%%%%%%%%%%%%%%%%%%%%%%%%%%%%%%%%%%

% \title{\LARGE \bf
% Preparation of Papers for IEEE Sponsored Conferences \& Symposia*
% }

% \author{Albert Author$^{1}$ and Bernard D. Researcher$^{2}$% <-this % stops a space
% \thanks{*This work was not supported by any organization}% <-this % stops a space
% \thanks{$^{1}$Albert Author is with Faculty of Electrical Engineering, Mathematics and Computer Science,
%         University of Twente, 7500 AE Enschede, The Netherlands
%         {\tt\small albert.author@papercept.net}}%
% \thanks{$^{2}$Bernard D. Researcheris with the Department of Electrical Engineering, Wright State University,
%         Dayton, OH 45435, USA
%         {\tt\small b.d.researcher@ieee.org}}%
% }

\title{\LARGE \bf
Conformal Koopman for Embedded Nonlinear Control with Statistical Robustness: Theory and Real-World Validation
}
\author{Koki Hirano and Hiroyasu Tsukamoto
\thanks{This work benefited from technical discussions within DARPA’s Safe and Assured Foundation Robots for Open eNvironments (SAFRON) Program, under contract number HR0011-25-3-0331. The authors are with the Department of Aerospace Engineering, The Grainger College of Engineering, University of Illinois Urbana-Champaign, Urbana, Illinois 61801, {\tt\small khirano@illinois.edu, hiroyasu@illinois.edu}.
}
}
%%%%%%%%%%%%%%%%%%%%%%%%%%%%%%%%%%%%%%%%%%%%%%%%
%%%%%%%%%%%%%%%%%%%%%%%%%%%%%%%%%%%%%%%%%%%%%%%%
%%%%%%%%%%%%%%%%%%%%%%%%%%%%%%%%%%%%%%%%%%%%%%%%
% main document
%%%%%%%%%%%%%%%%%%%%%%%%%%%%%%%%%%%%%%%%%%%%%%%%
%%%%%%%%%%%%%%%%%%%%%%%%%%%%%%%%%%%%%%%%%%%%%%%%
%%%%%%%%%%%%%%%%%%%%%%%%%%%%%%%%%%%%%%%%%%%%%%%%

\maketitle
\begin{tikzpicture}[remember picture,overlay]
\node[anchor=south,yshift=10pt] at (current page.south) {
    \parbox{0.9\textwidth}{
        \footnotesize \textcopyright 2026 IEEE. Personal use of this material is permitted. Permission from IEEE must be obtained for all other uses, in any current or future media, including reprinting/republishing this material for advertising or promotional purposes, creating new collective works, for resale or redistribution to servers or lists, or reuse of any copyrighted component of this work in other works.
    }
};
\end{tikzpicture}
\begin{tikzpicture}[remember picture,overlay]
\node[anchor=north,yshift=-30pt] at (current page.north) {
    \parbox{0.9\textwidth}{
        \centering \footnotesize
        This work has been accepted for publication in ICRA 2026. The final published version will be available via IEEE Xplore.
    }
};
\end{tikzpicture}
\thispagestyle{empty}
\pagestyle{empty}
\begin{abstract}
We propose a fully data-driven, Koopman-based framework for statistically robust control of discrete-time nonlinear systems with linear embeddings. Establishing a connection between the Koopman operator and contraction theory, it offers distribution-free probabilistic bounds on the state tracking error under Koopman modeling uncertainty. Conformal prediction is employed here to rigorously derive a bound on the state-dependent modeling uncertainty throughout the trajectory, ensuring safety and robustness without assuming a specific error prediction structure or distribution. Unlike prior approaches that merely combine conformal prediction with Koopman-based control in an open-loop setting, our method establishes a closed-loop control architecture with formal guarantees that explicitly account for both forward and inverse modeling errors. Also, by expressing the tracking error bound in terms of the control parameters and the modeling errors, our framework offers a quantitative means to formally enhance the performance of arbitrary Koopman-based control. We validate our method both in numerical simulations with the Dubins car and in real-world experiments with a highly nonlinear flapping-wing drone. The results demonstrate that our method indeed provides formal safety guarantees while maintaining accurate tracking performance under Koopman modeling uncertainty.
\end{abstract}
\section{Introduction}
% \begin{itemize}
%     \item Flapper~\cite{flapper_original}
%     \item Koopman and contraction~\cite{koopman_contraction}
% \end{itemize}
Data-driven control with the Koopman operator has shown potential in applying linear control techniques to nonlinear dynamical systems. By modeling the operator with finite-dimensional representations, the original nonlinear system can be lifted into a higher-dimensional latent space where the dynamics evolutions are approximately linear. In the lifted space, well-established linear control methods can be systematically applied to design computationally efficient and interpretable control policies \cite{Korda2018, Lusch2018}.

Recent studies have employed deep neural networks to learn the Koopman operator and its lifting functions \cite{Lusch2018, takeishi2017learning, wehmeyer2018time, morton2018deep}. However, two primary sources of modeling error persist: (i) the learned latent dynamics retain residual nonlinearity, as a finite-dimensional set of observables cannot perfectly linearize arbitrary nonlinear dynamics, and (ii) an encoder--decoder mismatch arises since the decoder is not explicitly constrained to be the inverse of the encoder. Such model approximation errors can critically impact control performance, especially in applications that require formal safety or stability guarantees. Existing Koopman-based control frameworks typically address these issues through assumptions on the modeling error and prediction structure. Some approaches derive Lipschitz bounds from training residuals to design robust model predictive control (MPC) formulations \cite{Mamakoukas2022}, while others incorporate auxiliary controllers or ensemble-based uncertainty modeling to compensate for latent space inaccuracies \cite{Zhang2022Automatica, Wang2023IJMLC}. These structural assumptions, however, are often difficult to justify and may not hold in practice.

\begin{figure}[t]
    \centering
    \includegraphics[width=1\linewidth]{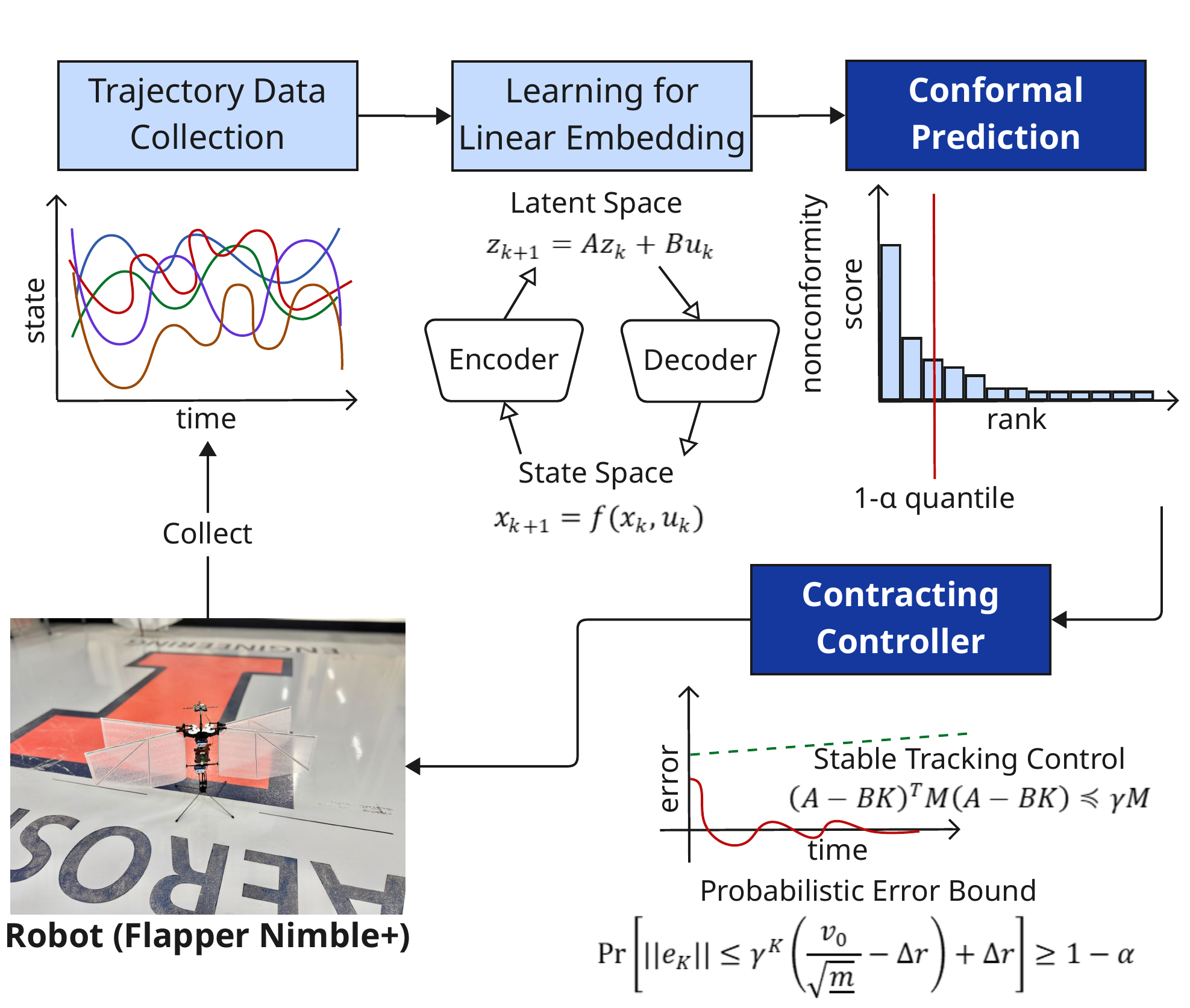}
    \caption{Overview of the conformal Koopman for embedded nonlinear control with statistical robustness.}
    \label{fig:concept}
    \vspace{-2em}
\end{figure}

To address these challenges of residual nonlinearity and encoder--decoder mismatch in a principled, data-driven manner, we propose a framework that integrates conformal prediction into \emph{closed-loop} Koopman-based control. Conformal prediction is a statistical technique that provides distribution-free, finite-sample guarantees on prediction errors under the mild assumption of exchangeable data \cite{conformal1, conformalbook}. We leverage this technique to move beyond empirical, open-loop prediction, formally quantifying the combined effect of these modeling errors on the closed-loop system. This provides a statistically rigorous probabilistic bound on the tracking error induced by a Lyapunov-based feedback control, even with the imperfectly linearized latent space. The connection between latent-space stability and stability of original discrete-time nonlinear systems is also established through the lens of contraction theory. Our main contributions are as follows:
\begin{enumerate}
\item We propose a statistical method to quantify the closed-loop tracking error in Koopman-based control, capturing residual nonlinearity and encoder--decoder mismatch via distribution-free guarantees of conformal prediction.
\item We establish a quantitative guideline for controller design by explicitly relating the closed-loop tracking performance to key system parameters (e.g., control gains and model uncertainty), enabling the theoretically formal design of learning-based Koopman controllers.
\item We apply this method to provide formal, probabilistic guarantees on tracking performance for discrete-time nonlinear systems controlled by Lyapunov-based feedback, connecting latent-space stability with stability in the original systems via contraction theory.
\item We validate our framework both in simulations and hardware experiments on a flapping-wing robot (\emph{Flapper Nimble+ from Flapper Drones}, \myhref{https://flapper-drones.com/}{https://flapper-drones.com/}), demonstrating its practical applicability in providing reliable performance guarantees for highly nonlinear systems\added{, where our controller serves as a high-level tracking layer that outputs target positions to an onboard PID}.
\end{enumerate}

\subsubsection*{Related Work}
%\subsection{Koopman Operator-Based Control}

The Koopman operator provides a linear but infinite-dimensional representation of nonlinear dynamical systems. For their practical applications, approximations in finite-dimensional subspaces are typically employed---e.g., via extended dynamic mode decomposition (eDMD) or neural networks---which enable the use of linear techniques in the analysis and control of nonlinear dynamical systems. For instance, \cite{Lusch2018} proposed deep autoencoders to learn Koopman-invariant subspaces and apply linear MPC in the latent space. Robustness to modeling errors has been studied in \cite{Mamakoukas2022ACC}, where a robust MPC is constructed under the assumption that the residual modeling error is Lipschitz continuous. However, this assumption can be restrictive, as strong nonlinearities may arise from both the underlying system dynamics and the neural networks often used to learn the lifting functions. 

Probabilistic finite-sample error bounds were proposed in~\cite{Nuske2023}, offering some theoretical guarantees for control systems approximated via Koopman embeddings with eDMD. Other approaches to compensating for the imperfect linearity include, but are not limited to, tube-based MPC \cite{Zhang2022Automatica}, polytopic uncertainty modeling  \cite{Wang2023IJMLC}, and stochastic methods that model uncertainty explicitly in the latent space \cite{Han2022ICLR}. However, many of these approaches rely on structural or distributional assumptions about the model prediction error. This limits their applicability in safety-critical scenarios where such assumptions may not be entirely valid.
%\subsection{Conformal Prediction for Control}

Conformal prediction \cite{conformal1, conformalbook} has recently emerged as a powerful framework for uncertainty quantification in data-driven systems. It operates under the assumption of exchangeability, i.e., the joint probability distribution of sampled data is invariant under permutations, thereby relaxing the typical i.i.d. requirement. When applied to dynamical system modeling, this property eliminates the need for structural or distributional priors, making it a promising tool for designing interpretable data-driven control~\cite{2024_Lindemann_CP-control-survey}. This concept has been applied in diverse areas of control theory, including safe motion planning~\cite{2023_Lindemann_CP-Planning, 2023_Dixit_ACP-Planning, 2023_Sun_CP-Diffusion}, robust control of linear systems~\cite{Patel2024Conformal,CPLinearStochastic}, safe learning-based MPC~\cite{2023_Chee_WPC-MPC,confcont3,2024_Zhou_ACP-CBF-MPC}, and Lyapunov-based analysis~\cite{2023_Yang_CP-Sensor, TingWeiCProbust, SWei_conformalContraction}.
%  % It allows for constructing prediction intervals or sets that contain the true outcome with a user-specified probability, without assuming a particular noise distribution.
% In the context of control, conformal prediction has been used to design safe planning frameworks and filtering mechanisms for reinforcement learning agents \cite{Patel2024Conformal}. These methods typically apply conformal prediction to estimate bounds on trajectory prediction errors, enabling the enforcement of safety constraints under uncertainty.

Conformal prediction is also explored in Koopman-based control~\cite{liang2025safenavigationdynamicenvironments} to derive a statistical bound on one-step ahead state prediction errors in Koopman-based MPC. While it addresses modeling error in the state space in an open-loop, single-step prediction setting, our work focuses on the closed-loop, trajectory-wise control performance degradation caused by imperfectly learned latent spaces. In particular, we explicitly consider aggregated errors from residual nonlinearity and the encoder--decoder mismatch, the latter of which arises from the decoder being an imperfect inverse of the encoder. Our key contribution is to establish a formal statistical characterization of closed-loop tracking error and robustness in such imperfectly lifted spaces. This is supported by its hardware validation with Lyapunov-based control, with the connection between latent-space stability and stability of original systems implied via contraction theory.

\section{PRELIMINARIES AND PROBLEM FORMULATION}
\label{sec_problem}
%\subsection{System Dynamics}
We consider the following discrete-time nonlinear control system indexed by the time step $k$:
\begin{equation}
    x_{k+1} = f(x_k, u_k),
    \label{eq:nonlinear_system}
\end{equation}
where $x_k \in \mathbb{R}^n$ is the state, $u_k \in \mathbb{R}^m$ is the control input, and $f: \mathbb{R}^n \times \mathbb{R}^m \rightarrow \mathbb{R}^n$ is an \emph{unknown} nonlinear function.

\subsection{Koopman Operator Framework}
The Koopman operator is a linear operator that describes the evolution of observation functions of the state. By lifting the state to a higher-dimensional space of observables, the nonlinear dynamics in the original state space can be represented by linear dynamics in the lifted space.

Let $\mathcal{G}$ be the space of observable functions $g: \mathbb{R}^n \rightarrow \mathbb{R}$. The Koopman operator $\mathcal{K}$ acts on an observable $g \in \mathcal{G}$ as:
\begin{equation}
    (\mathcal{K}g)(x) = g(f(x, u)).
\end{equation}
In practice, we use a finite-dimensional approximation of the Koopman operator. We choose a vector of $N$ basis functions $\hat{g}: \mathbb{R}^n \rightarrow \mathbb{R}^N$ to lift the state $x_k$ to a higher-dimensional state $z_k = \hat{g}(x_k)$. This allows us to approximate the system~\eqref{eq:nonlinear_system} by a linear system with a residual error term.

% Definition 1: Koopman Lifted System
\begin{definition}[Koopman Lifted System]
The system \eqref{eq:nonlinear_system} is approximated by the following lifted linear system with a residual error $d(x_k, u_k)$:
\begin{equation}
    z_{k+1} = A z_k + B u_k + d(x_k, u_k)
    \label{eq:lifted_system}
\end{equation}
where $z_k = \hat{g}(x_k) \in \mathbb{R}^N$ is the lifted state. The matrices $A \in \mathbb{R}^{N \times N}$ and $B \in \mathbb{R}^{N \times m}$ are learned from data such that the pair $(A, B)$ is controllable.
\end{definition}

This approximation enables the full utilization of linear control theory for analyzing and designing control policies. The residual term $d(x_k, u_k)$ denotes the difference between the lifted next observed state $\hat{g}(x_{k+1})$ and the linear propagation of the previous observed state $(A z_k + B u_k)$. This will be quantified using conformal prediction in our method. 
%Our goal is to design a controller that is robust to this error.  

\subsection{Split Conformal Prediction}
Given a dataset split into a training set and a calibration set, let us consider a point predictor $\widehat{m}(\cdot)$ fitted on the training set.  
For each point $(x_i,y_i)$ in the calibration set, we compute the nonconformity score as
\begin{align}
R_i \coloneqq \|y_i - \widehat{m}(x_i)\|, \quad i=1,\dots,m.
\end{align}
Let $R_{(1)} \le \cdots \le R_{(m)}$ be the ordered residuals and define $k \coloneqq \lceil (m+1)(1-\alpha) \rceil$. 
If the calibration scores are exchangeable with the score of a newly acquired point $(x_{m+1},y_{m+1})$, then we have the following~\cite{conformal1, conformalbook}:
\begin{align}
\Pr[\,Y_{m+1} \in \hat{C}(X_{m+1})]
\;\ge\; \tfrac{\lceil (m+1)(1-\alpha)\rceil}{\,m+1\,}
\ge 1-\alpha
\end{align}
where the prediction set is given by
\begin{align}
\hat{C}(x_{n+1}) = 
\bigl[\;\widehat{m}(x_{n+1}) - R_{(k)},\;\widehat{m}(x_{n+1}) + R_{(k)}\;\bigr].
\end{align}

\section{MAIN CONTRIBUTIONS}
\label{sec_tracking}
In this section, we present our primary result: a robust tracking controller and a predictive planner designed in the Koopman lifted space. We first analyze the connection between stability in the Koopman lifted space and stability in our original system using contraction theory. We then delineate the controller design and derive high-probability bounds on the tracking error. Finally, we \replaced{derive high-probabilistic guarantees on safety}{design a predictive planner that has high-probabilistic guarantees on safety}.
\subsection{Controller Design}
We design two types of tracking controllers in the lifted system \eqref{eq:lifted_system}. Here, we denote a reference trajectory by $(x_{d,k}, u_{d,k})$, with its lifted counterpart in the latent space computed by $z_{d,k} = \hat{g}(x_{d,k})$. \added{The subscript $d$ is used to denote the desired states and inputs.}  
The target system is represented as:
\begin{align}
z_{d, k+1} = Az_{d, k} + Bu_{d, k} + d'_k \left(z_{d,k}, u_{d,k}\right)
\label{eq:ref_lifted_system}
\end{align}
where $d'_k: \mathbb{R}^p\times \mathbb{R}^m \mapsto \mathbb{R}^p$ is an unknown function denoting the difference between the lifted next target state $\hat{g}(x_{d,k+1})$ and the linear propagation of the previous target state. 

\subsubsection{Nominal Feedback Controller}
Let us first define a Lyapunov-based feedback controller for trajectory tracking.
\begin{definition}[Nominal Feedback Controller (NFC)]
\label{def:nominal_feedback}
Our nominal feedback controller is defined as:
\begin{equation}
    u_{k} = u_{d,k} - Ke_k.
    \label{eq:nominal_controller}
\end{equation}
Here, the gain matrix $K$ is designed such that for a given positive definite matrix $M=\Theta^T\Theta\succ0$ satisfying $\underline{m}\mathbb{I}\preceq M\preceq \overline{m}\mathbb{I}$, the closed-loop matrix $A_{cl} = A - BK$ satisfies the following Lyapunov condition:
\begin{equation}
A_{cl}^\top M A_{cl} \preceq \gamma M
\label{eq:lmi}
\end{equation}
for a given decay rate $\gamma \in (0,1)$.
\end{definition}

By defining $\hat{d}_k(z_k, u_k, z_{d,k}, u_{d,k}) = d_k(z_k, u_k) - d'_k(z_{d,k}, u_{d,k})$, the dynamics of the tracking error $e_k \coloneqq z_k - z_{d, k}$ can be written as:
\begin{align}
    e_{k+1} = A_{cl}e_k + \hat{d}_k(z_k, u_k, z_{d,k}, u_{d,k}).
    \label{eq:error_dynamics}
\end{align}

\subsubsection{Controller with Robust Disturbance Rejection}
While the nominal controller ensures stability for the nominal system, its performance can be degraded by the residual dynamics $\hat{d}_k$. To enhance robustness, we formulate a controller that explicitly mitigates such disturbances by solving an optimization problem at each time step. This is to maintain stability while minimizing deviation from a desired control input. We now define the Controller with Robust Disturbance Rejection (CRDR).
\begin{definition}[CRDR]
The control input $u_k$ is determined by solving the following convex optimization problem at each time step $k$:
\begin{align}
\label{eq:robust_constraint}
\begin{aligned}
    &(u_k, \Delta_{v, k}) = \arg \min_{(u_k, \Delta_{v, k})} \left\| \Delta u_k \right\|^2 + c_v \Delta_{v, k}^2\\
% \label{eq:robust_opt_problem}
&\text{s.t.} \quad \left\| \Theta (A e_k + B \Delta u_k) \right\| \le \gamma \left\| \Theta e_k \right\| - \rho + \Delta_{v, k}
\end{aligned}
\end{align}
where $\Delta u_k \coloneqq u_k-u_{d,k}$. Here, $v_k\coloneqq||\Theta (z_k - z_{d,k})||$ is a Lyapunov function, with $\Theta$ obtained by solving the Lyapunov inequality for linear systems in~\eqref{eq:lmi}.
%$\Theta$ is computed by the linear matrix inequality~\eqref{eq:lmi}.
The constraint~\eqref{eq:robust_constraint} enforces contraction on the nominal dynamics. The positive constant $\rho$ serves as a robustness margin against the disturbance $\hat{d}_k$, while the slack variable $\Delta v_k$ ensures feasibility of the optimization even under unforeseen disturbances. The coefficient $c_v > 0$ is a penalty weight on the slack variable.    
\label{def:disturbance_rejection}
\end{definition}

\subsection{Connection to Contraction Theory}
Our control design is founded on the principle that stabilizing the system in the lifted space robustly enforces stability in the original nonlinear state space. This is formalized by the equivalence between Koopman stability and contraction, as established in~\cite{koopman_contraction} at least for continuous-time systems. We extend their work to discrete-time systems with the nominal feedback controller defined in Definition \ref{eq:nominal_controller}.

\begin{theorem}
If the Jacobian of the lifting map, $\hat{G}(x_k) = \tfrac{\partial \hat{g}}{\partial x}(x_k)$, has full column rank, then the matrix $W(x_k) = \hat{G}(x_k)^\top M \hat{G}(x_k)$ is a valid contraction metric for the original nonlinear system, i.e.,
\begin{equation}
    \label{eq_contraction_koopman}
    \tfrac{\partial f_{cl}}{\partial x}(x_k)^T W(f_{cl}(x_k))\tfrac{\partial f_{cl}}{\partial x}(x_k)
    \preceq \gamma W(x), \quad \forall x_k\in \mathbb{R}^n.
\end{equation}
\end{theorem}

\begin{proof}
Let $f_{cl}(x_k)\coloneqq f(x_k, u_k) = x_{k+1}$ where $u_k$ is given by \eqref{eq:nominal_controller}. 
From the definition of linear dynamics in the latent space, the following equation holds:
\begin{align}
    e_{k+1}=A_{cl}e_k,~~\hat{g}(x_{k+1}) - z_{d,k+1}=A_{cl}(\hat{g}(x_k)-z_{d,k}).
\end{align}
Differentiating both sides with respect to $x$, we get 
\begin{align}
    \hat{G}(f_{cl}(x_k))\tfrac{\partial f_{cl}}{\partial x}(x_k) = A_{cl}\hat{G}(x_k)
\end{align}
which implies
\begin{align}
    &\tfrac{\partial f_{cl}}{\partial x}(x_k)^T W(f_{cl}(x_k))\tfrac{\partial f_{cl}}{\partial x}(x_k)\\
    &= \tfrac{\partial f_{cl}}{\partial x}(x_k)^T\hat{G}(x_k)^\top M \hat{G}(x_k)\tfrac{\partial f_{cl}}{\partial x}(x_k)\\
    &=G(x)^TA_{cl}^TMA_{cl}G(x_k).
\label{eq:thm1-1}
\end{align}
Since multiplying $\hat{G}(x)^T$ from the left side and $\hat{G}(x)$ from the right side on \eqref{eq:lmi} gives $\hat{G}(x_k)^T A_{cl}^T M A_{cl} \hat{G}(x_k) \preceq \gamma \hat{G}(x_k)^TM\hat{G}(x_k)$, the relation \eqref{eq:thm1-1} results in the inequality~\eqref{eq_contraction_koopman} as desired.
\end{proof}

\subsection{High-Probability Error Bounds in Latent Space}

While the nominal controller ensures stability for the idealized, error-free system, its performance in practice is heavily affected by the residual dynamics error $d(x_k, u_k)$. To guarantee system safety and robustness, we leverage conformal prediction to establish high-probability bounds on this residual error in the latent space. The choice of the nonconformity score, a key element in conformal prediction, depends on the specific controller being analyzed. In the following, we define two distinct nonconformity scores: one for the nominal feedback controller in Definition~\ref{def:nominal_feedback} and another for the robust controller in Definition~\ref{def:disturbance_rejection}.

\begin{definition}[Forward Nonconformity Score for NFC]
Let $\Delta_{d,k} \coloneqq \|\hat{d}_k(z_k, u_k, z_{d,k}, u_{d,k})\|$ for the residual error in~\eqref{eq:error_dynamics}. The nonconformity score is defined as:
\begin{align}
s_{\text{fwd}, k} &\coloneqq \Delta_{d,k} = \|d_k(z_k, u_k) - d'_k(z_{d,k}, u_{d,k})\| \nonumber \\
&= \|\hat{g}(x_{k + 1}) - (A\hat{g}(x_k) +Bu_k) \\
& \quad \quad - \hat{g}(x_{d, k + 1}) + (A\hat{g}(x_{d,k}) +Bu_{d,k})\|.
\label{eq:fwd_score_nom}
\end{align}
Here, the states $x_k$ and $x_{k+1}$ are obtained by sampling state transitions from the real-world system. The desired states $x_{d,k}$ and $x_{d,k+1}$ are computed from the reference trajectory and corresponding control inputs generated by the planner.
\label{def:fwd_cp_score_nfc}
\end{definition}

\begin{definition}[Forward Nonconformity Score for CRDR]
Let $\Delta_{d,k} \coloneqq \|\hat{d}_k(z_k, u_k, z_{d,k}, u_{d,k})\|$ for the residual error in~\eqref{eq:error_dynamics}. The nonconformity score is defined as:
\begin{equation}
s_{\text{fwd}, k}= \sqrt{\overline{m}}\Delta_{d,k}+\Delta_{v, k}
\label{eq:fwd_score}
\end{equation}
where $\sqrt{\overline{m}}$ denotes the maximum singular value of $\Theta$ in~\eqref{eq:robust_constraint}. The term $\Delta_{v,k}$, defined also in~\eqref{eq:robust_constraint}, is determined by applying the controller from Definition \ref{def:disturbance_rejection} to a target trajectory sampled from the same distribution as the deployment environment. For practical application to a physical system, a surrogate model of the dynamics can first be identified from collected data. This model is then used to simulate the tracking of desired trajectories, allowing the data sampling of $\Delta_{v,k}$.
\label{def:fwd_cp_score_robust}
\end{definition}

\begin{theorem}
Suppose that the forward nonconformity scores in Definitions~\ref{def:fwd_cp_score_nfc} and~\ref{def:fwd_cp_score_robust} are exchangeable with the unseen, future score $s_{\text{fwd}}$ from the real world, i.e., they satisfy the following inequality via conformal prediction:
\begin{equation}
\Pr[s_{\text{fwd}} \leq q_{1-\delta}]\geq1-\delta
\label{eq:cp_ineq}
\end{equation}
where $\delta=\tfrac{\alpha}{K}\in(0,1)$ is a user-defined failure probability, and $q_{1-\delta}$ is the $(1-\delta)$th quantile of the empirical distribution of past scores. Then the tracking error in latent space is exponentially bounded with a probability of at least $1-\alpha$ for each time step $k$ as follows:
\begin{equation}
    \Pr\left[||e_{k+K}||\leq\gamma^K\left(\tfrac{v_k}{\sqrt{\underline{m}}}-\Delta r\right)+\Delta r\right]\geq 1-\alpha
\label{eq:thm2}
\end{equation}
where for the controller defined in Definition \ref{def:nominal_feedback},
\begin{equation}
    \Delta r\coloneqq\tfrac{\sqrt{\overline{m}}q_{1-\delta}}{(1-\gamma)\sqrt{\underline{m}}}
\end{equation}
for the controller defined in Definition \ref{def:disturbance_rejection},
\begin{equation}
    \Delta r\coloneqq\tfrac{q_{1-\delta}-\rho}{(1-\gamma)\sqrt{\underline{m}}}.
\end{equation}
This bound implies that if $\tfrac{v_0}{\sqrt{\underline{m}}}-\Delta r \leq 0$, we have
\begin{equation}
    \Pr\left[||e_{k+K}||\leq \Delta r\right]\geq 1-\alpha.
\end{equation}
Here, $\sqrt{\overline{m}}$ and $\sqrt{\underline{m}}$ denote the maximum and minimum singular values of $\Theta$ in~\eqref{eq:robust_constraint}.
\label{thm:2}
\end{theorem}
\begin{proof}
    Consider the Lyapunov function in~\eqref{eq:robust_constraint}. Its time evolution along the dynamics \eqref{eq:lifted_system} and \eqref{eq:ref_lifted_system} with the controller of Definition~\ref{def:disturbance_rejection} can be computed as
    \begin{align}
    \begin{aligned}        
        v_{k+1}&=\|\Theta (z_{k+1} - z_{d, k+1})\|\\
        &=\left\|\Theta\left(Ae_k+B\Delta u_k + \hat{d}_k(z_k, u_k, z_{d,k}, u_{d,k})\right)\right\|\\
        &\leq\gamma\|\Theta e_k\| - \rho + \Delta_{v, k} + \|\Theta \hat{d}_k(z_k, u_k, z_{d,k}, u_{d,k})\|.
    \end{aligned}
    \label{eq:thm2-1}    
    \end{align}
    Since the last term can be bounded as $\|\Theta \hat{d_k}\|\leq \sqrt{\overline{m}}\|\hat{d}_k\|$, the inequality \eqref{eq:thm2-1} can be written as:
    \begin{align}
        v_{k+1}\leq \gamma\|\Theta e_k\| - \rho + \Delta_{v, k} +\sqrt{\overline{m}}\hat{d_k}.
    \label{eq:thm2-2}
    \end{align}
    Unrolling the recursion in \eqref{eq:thm2-2} from $k=0$ to $K-1$ yields
    \begin{align}
        v_K &\leq \gamma^Kv_0 + \sum_{k=0}^{K-1} \gamma^{K-1-k}\left(- \rho + \Delta_{v, k} +\sqrt{\overline{m}}\hat{d_k}\right)\\
        &\leq\gamma^Kv_0 - \rho \tfrac{1-\gamma^K}{1-\gamma} + \sum_{k=0}^{K-1} \gamma^{K-1-k}s_{\text{fwd}, k}.
    \end{align}
    The probability that \eqref{eq:cp_ineq} holds for all steps $k=0,...,K-1$ can be computed by the union bound as follows:
    \begin{align}
        \Pr\left[\cup_{k=0}^{K-1}{s_k>q_{1-\delta}}\right] \leq \sum_{k=0}^{K-1}\Pr[s_k >q_{1-\delta}] \leq \sum_{k=0}^{K-1}\tfrac{\alpha}{K} = \alpha
    \end{align}
    Under this high-probability event, we obtain the following high-probability bound on the Lyapunov function:
    \begin{align}
        \Pr\left[v_k\leq\gamma^Kv_0 + \left(q_{1-\delta} - \rho\right)\tfrac{1-\gamma^K}{1-\gamma}\right]\geq 1-\alpha.
    \label{eq:lyapunov_prob_ineq}
    \end{align}
    Finally, we relate the Lyapunov function $v_k$ back to the tracking error $\|e_k\|$.
    From $M\succeq \underline{m}\mathbb{I}$, we get
    \begin{align}
        v^2_K=e_K^TMe_k\geq\underline{m}\|e_K\|^2
    \end{align}
    which implies $\|e_K\|\leq\tfrac{v_K}{\sqrt{\underline{m}}}$.
    Thus, \eqref{eq:lyapunov_prob_ineq} reduces to
    \begin{align}
        \Pr\left[\|e_K\|\leq \gamma^K\tfrac{v_0}{\sqrt{\underline{m}}} + \tfrac{q_{1-\delta} - \rho}{(1-\gamma)\sqrt{\underline{m}}}(1-\gamma^K)\right]\geq 1-\alpha
    \end{align}
    Using the definition $\Delta r\coloneqq \tfrac{q_{1-\delta}-\rho}{(1-\gamma)\sqrt{\underline{m}}}$,
    \begin{equation}
        \Pr\left[||e_K||\leq\gamma^K\left(\tfrac{v_0}{\sqrt{\underline{m}}}-\Delta r\right)+\Delta r\right]\geq 1-\alpha
    \end{equation}
    By setting $\rho=\Delta_{v, k}=0$ in the previous discussion, we can also prove \eqref{eq:thm2} for the nominal feedback controller with $\Delta r = \tfrac{\sqrt{\overline{m}}q_{1-\delta}}{(1-\gamma)\sqrt{\underline{m}}}$
\end{proof}

The slack variable $\Delta_{v,k}$ within the score of Definition \ref{def:fwd_cp_score_robust} is an outcome of the online optimization performed by the controller at each step. Therefore, to obtain a distribution of these scores for conformal prediction before deployment, one must rely on extensive simulations with a surrogate model. We could circumvent this requirement by using the simpler score from Definition \ref{def:fwd_cp_score_nfc} also for the controller in Definition~\ref{def:disturbance_rejection}. In this case, we can derive an error bound by retaining the explicit summation of the slack variables $\Delta_{v,k}$. Specifially, under the high-probability event where $||\hat{d}_k|| \le q_{1-\delta}$ for all $k \in \{0, \dots, K-1\}$, we get
$$v_K \le \gamma^K v_0 + (-\rho + \sqrt{\overline{m}}q_{1-\delta})\tfrac{1-\gamma^K}{1-\gamma} + \sum_{k=0}^{K-1}\gamma^{K-1-k}\Delta_{v,k}.$$
\begin{corollary}[Trajectory-Dependent Error Bound]
When employing CRDR of Definition~\ref{def:disturbance_rejection} with the nominal nonconformity score of Definition \ref{def:fwd_cp_score_nfc}, the probabilistic bound on the tracking error $||e_K||$ is given by:
\begin{align} \label{eq:traj_dep_bound}
Pr\left[||e_K|| \le \tfrac{1}{\sqrt{\underline{m}}} \left( \gamma^K v_0 + \tfrac{1-\gamma^K}{1-\gamma}(-\rho + \sqrt{\overline{m}}q_{1-\delta})\right.\right. \nonumber \\
\left.\left. + \sum_{k=0}^{K-1}\gamma^{K-1-k}\Delta_{v,k} \right)\right] \ge 1-\alpha
\end{align}
\label{cor:trajectory_dep}
\end{corollary}

The error bound is no longer a simple fixed-radius ball, but rather a trajectory-dependent bound that depends on the history of $\Delta_{v,0}, \dots, \Delta_{v,K-1}$ generated by the controller.

\subsection{High-Probability Error Bounds in State Space}
Finally, we consider the error bound in the state space when the tracking error in latent space is bounded as in Theorem \ref{thm:2}.  When the lifting and its inverse mapping are trained separately, e.g., as neural networks, the inverse mapping $\hat{g}_{\text{inv}}$ may not be equal to the inverse of the lifting function $\hat{g}^{-1}$. We define a nonconformity score to quantify such a discrepancy.
\begin{definition}[Round-trip Nonconformity Score]
    The \replaced{`round-trip'}{``round-trip''} nonconformity score is defined as:
    \begin{equation}
        s_k \coloneqq \|x_k-\hat{g}_{\text{inv}}(\hat{g}(x_k))\|
    \end{equation}
\end{definition}
\begin{theorem}[State-Space Tracking Error Bound]
    Suppose that the following assumptions hold:
    \begin{enumerate}
        \item The conditions of Theorem \ref{thm:2} hold, i.e., the tracking error is bounded as $\Pr\left[||e_k||\leq\epsilon_K\right]\geq 1-\alpha$, where $\epsilon_K = \gamma^K\left(\tfrac{v_0}{\sqrt{\underline{m}}}-\Delta r\right)+\Delta r$
        \item $1-\beta/2$ quantile of the round-trip nonconformity score $q_{\text{rt}}$ provides the guarantee $\Pr[\|x_K-\hat{g}_{\text{inv}}(\hat{g}(x_K))\|\leq q_{\text{rt}}]\geq 1-\beta/2$. \label{asp:1}
        \item $\hat{g}_{\text{inv}}$ is $L$-Lipschitz continuous with constant $L>0$. \label{asp:2}
        \item The desired state-space trajectory $x_{d,K}$ is given, and its lifted-space counterpart is computed as $z_{d,K} = \hat{g}(x_{d,K})$.
    \end{enumerate}
    Then for each time step $k$, we have the following:
    \begin{equation}
        \Pr\left[||x_{k+K} - x_{d,k+K}||\leq 2q_{\text{rt}} + L\epsilon_K\right]\geq 1-\alpha-\beta.
    \end{equation}

\end{theorem}
\begin{proof}
    By applying the triangle inequality, we obtain:
    \begin{align}
        \|x_K - x_{d,K}\|&\leq \|x_K-\hat{g}_{\text{inv}}(z_K)\|\nonumber
        +\|\hat{g}_{\text{inv}}(z_K)-\hat{g}_{\text{inv}}(z_{d,K})\|\\
        &+\|\hat{g}_{\text{inv}}(z_{d,K})-x_{d,K}\|
        \label{eq:thm3-1}
    \end{align}
    The first and third term are the round-trip error for the true state and the desired state.  By assumption \ref{asp:1}, both terms are independently bounded by the conformal quantile $q_{\text{rt}}$ each with a probability of at least $1-\beta/2$.  Let these be Events A and C.
    The second term is bounded using the $L$-Lipschitz property of $\hat{g}_{\text{inv}}$ (Assumption \ref{asp:2}).  From the premise of Theorem \ref{thm:2}, the lifted-space error is bounded by $\epsilon_K$ with a probability of at least $1-\alpha$. Let this be Event~B.
    The overall bound in inequality \eqref{eq:thm3-1} holds if Events~A, B, and C occur simultaneously. Using the union bound, the probability that at least one of these events fails is at most $\beta/2 + \alpha + \beta/2 =\alpha + \beta$. Thus, the joint event occurs with a probability of at least $1-\alpha-\beta$.
    When this joint event occurs, we substitute the individual bounds into inequality \eqref{eq:thm3-1}:
    \begin{align}
        \|x_K - x_{d,K}\|\leq q_{\text{rt}} + L\epsilon_K + q_{\text{rt}} = 2q_{\text{rt}} + L\epsilon_K
    \end{align}
    This inequality holds with a probability of at least $1-\alpha-\beta$, which completes the proof.
\end{proof}

This analysis provides a crucial insight: the tracking error bound is not solely determined by modeling error. Instead, it is explicitly shaped by key design parameters—the controller's Lyapunov matrix eigenvalues ($\overline{m}, \underline{m}$) and the decoder's Lipschitz constant $L$. This offers a guideline for designing Koopman controllers, directly linking these fundamental parameters to the achievable tracking precision.

\added{Conditions such as latent contraction, a full-rank Jacobian, and the Lipschitz property of the decoder can be enforced during training through appropriate loss functions. However, in our simulations, these conditions are verified \textit{a posteriori} for simplicity. Our empirical data suggest these conditions are largely satisfied within the calibration region. Enforcing these properties explicitly by design can be an important direction for future work.}
\section{Numerical Simulations}
\label{sec_simulation}
%%%%%%%%%%%%%%%%%%
%%%%%%%%%%%%%%%%%%

We first validate our proposed method using the Dubins car model, which exhibits nonholonomic behavior and nonlinear state evolution. The discrete-time dynamics of the Dubins car are given by $x_{k+1} = x_k + v_k \Delta t \cos \theta_k$, $y_{k+1} = y_K +v_k \Delta t \sin \theta$, $\theta_{k+1} = \theta_k + \omega \Delta t$, and $v_{k+1} = v_k + a_k\Delta t$. Here, $(x, y)$ is the position of the car, $\theta$ is its heading angle, $v$ is the forward velocity, $a$ is the acceleration, $\omega$ is the angular velocity, and $(a,\omega)$ is the control input. 

\subsection{Training}
The training process is conducted in two sequential phases. The first phase focuses on learning a latent representation, while the second identifies a linear dynamical model within that space.

\subsubsection{Data Collection}
Training data is generated in simulation with constant speed and random steering input sampled at 10-step intervals. We collect 1,000 episodes (100 steps each) for representation learning and 100 episodes for dynamics identification.

\subsubsection{Phase 1: Latent Representation Learning}
An encoder--decoder architecture is trained to find an effective latent representation of the system dynamics. The encoder, $g(\cdot)$, an MLP, maps an observation $x_k \in \mathbb{R}^{4}$ (vehicle position and orientation) to a latent state $z_k \in \mathbb{R}^{6}$. A corresponding MLP decoder, $h(\cdot)$, reconstructs the observation from this latent state. Both networks feature a 256-dimensional hidden layer with ReLU activation and Batch Normalization. This network is trained by minimizing a composite loss function, $\mathcal{L}_{\text{NN}}$, designed to encourage accurate reconstruction and dynamically consistent embeddings:
\[
\mathcal{L}_{\text{NN}} = \mathcal{L}_{\text{rss}} + \alpha \mathcal{L}_{\text{recon}} + \beta \mathcal{L}_{\text{ctrl}}
\]
where $\mathcal{L}_{\text{rss}}$ is the one-step prediction error in the latent space (using batch-wise linear approximations), $\mathcal{L}_{\text{recon}} = \mathbb{E}[ \| x_k - h(g(x_k)) \|_2^2 ]$ is the autoencoder's reconstruction error, and $\mathcal{L}_{\text{ctrl}}$ is a controllability penalty.

\subsubsection{Phase 2: Linear Dynamics Identification}
Once the encoder--decoder network is trained, the encoder's parameters are frozen. The learned mapping $z_k=g(x_k)$ is used to transform the entire dataset of observation sequences into a static set of latent state transitions. Subsequently, a single linear dynamical model, $z_{k+1} = Az_k + Bu_k$, is identified by finding the matrices A and B that minimize the following loss function, $\mathcal{L}_{\text{Koopman}}$:
\[
\mathcal{L}_{\text{Koopman}} = \mathcal{L}_{\text{pred}} + w_{\rho} \mathcal{L}_{\rho} + w_{\text{ctrl}} \mathcal{L}_{\text{ctrl}}
\]
Here, $\mathcal{L}_{\text{pred}} = \mathbb{E}[\|z_{k+1} - (Az_k + Bu_k)\|_2^2]$ is the primary mean squared prediction error. The loss also includes two regularization terms: $\mathcal{L}_{\rho} = \rho(A)$, which penalizes the spectral radius of A to encourage stability, and the same controllability loss, $\mathcal{L}_{\text{ctrl}} = -\log(\sigma_{\min}(\bm{C}) + \epsilon) + \lambda \tfrac{\sigma_{\max}(\bm{C})}{\sigma_{\min}(\bm{C}) + \epsilon}$, to ensure that the final linear model is robust and uniformly controllable.  The resulting one-step prediction accuracy is shown in Figure \ref{fig:prediction_state}

\begin{figure}[t]
    \centering
    \includegraphics[width=1\linewidth]{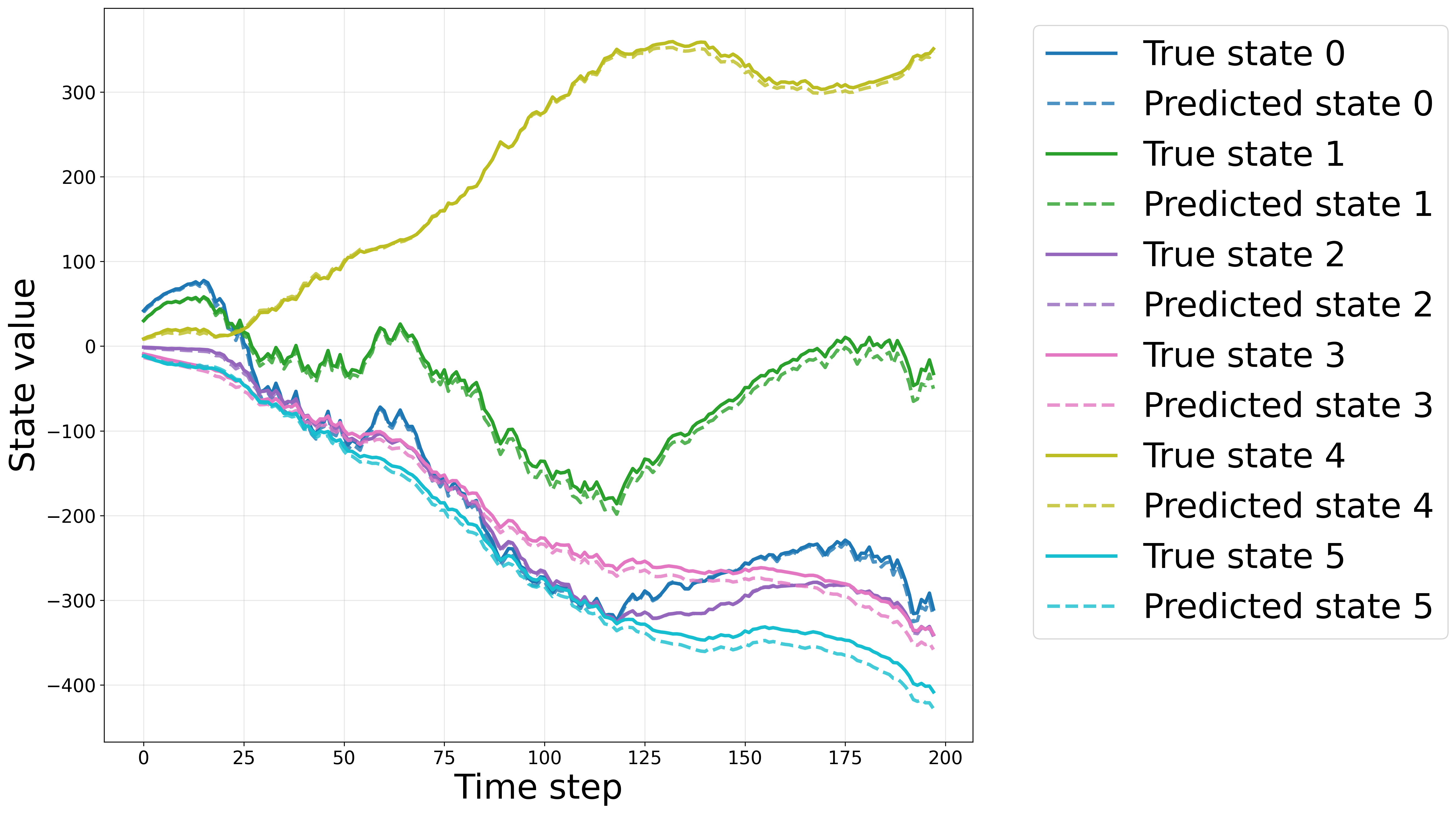}
    \caption{Latent space one-step prediction. The identified linear model accurately predicts the next latent state from the current state and control input. \added{Here, one-step prediction refers to propagating the latent state by one step and then decoding it back to the original space.}}
    \label{fig:prediction_state}
    \vspace{-1em}
\end{figure}
The control feedback gain $K$ for the nominal feedback controller is synthesized by solving a linear matrix inequality (LMI) to theoretically guarantee stability, as confirmed in Figure~\ref{fig:eigenvalues}. 

\begin{figure}[t]
    \centering
    \includegraphics[width=0.8\linewidth]{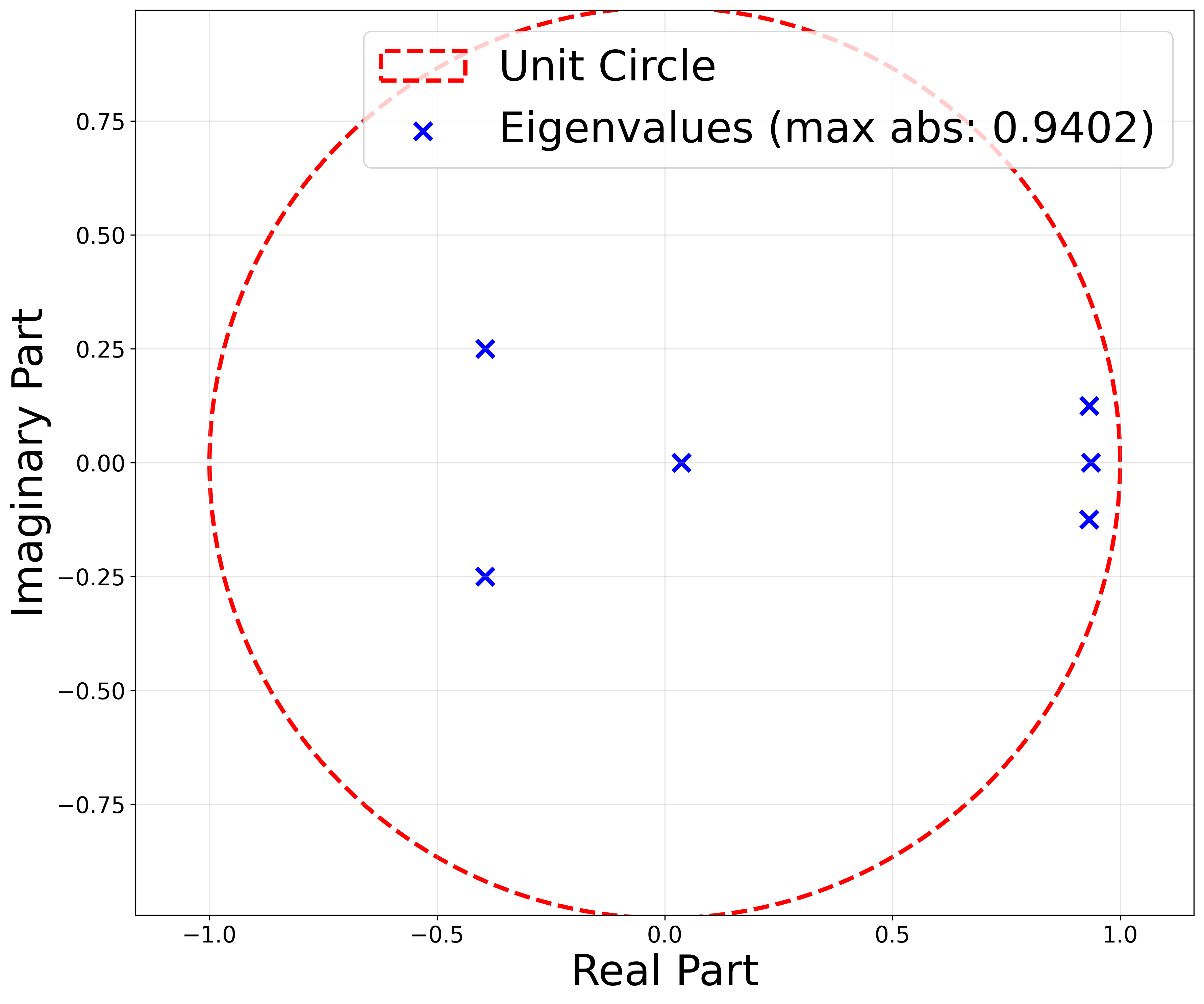}
    \caption{Closed-loop eigenvalues for the nominal feedback controller. The controller design yields a theoretically stable system, with all eigenvalues located inside the unit circle.}
    \label{fig:eigenvalues}
    \vspace{-1em}
\end{figure}
\subsection{Reference Trajectory Construction}
We construct a circular reference trajectory $\{x_{d,k}, u_{d,k}\}_{k=0}^{T-1}$ with radius $r>0$, centered at $(0,r)$, starting from the origin with desired speed $v_d>0$. The reference observation $x_{d,k}$ consists of position $(x_{\mathrm{ref}}, y_{\mathrm{ref}})$ and heading $(\sin\theta_d, \cos\theta_d)$, and the control input $u_{d,k}=\omega_k$ is the angular velocity computed by finite difference $\omega_k \approx \tfrac{\theta_d(t_{k+1}) - \theta_d(t_{k-1})}{2\Delta t}$ and clipped to $[-\pi, \pi]$.

\subsection{Evaluation}
We evaluate the controller performance on the circular trajectory tracking task described previously. The latent representation is trained with loss weights $\alpha=0.1$ and $\beta=0.1$. CRDR is designed with parameters $c_v = 0.01$, $\rho=0.073$ and $\gamma = 0.9$. We compare its performance against the nominal feedback controller (NFC). The tracking evaluation is conducted for a duration of 50 time steps.
\added{The controllability loss weight was the most critical hyperparameter: disabling it caused the condition number of the dynamics matrices to exceed $10^6$, leading to unstable control. The reconstruction weight was stable across $[0.05, 0.2]$, and the embedding dimension $N \approx 1.3n$ balanced expressiveness and controllability.}

\subsection{Results}
The results are presented in Figure \ref{fig:results}. Despite its theoretical stability, NFC consistently fails due to control inputs exceeding the actuator saturation limits ($\pm\pi$ rad/s), causing unbounded error accumulation. In contrast, CRDR achieves excellent tracking by finding inputs that satisfy Lyapunov stability conditions while compensating for the modeling error. Furthermore, the position error of CRDR remains well within the trajectory-dependent theoretical bound derived from Corollary \ref{cor:trajectory_dep}, validating our error analysis.

\begin{figure}[h!]
    \centering
    \includegraphics[width=1\linewidth]{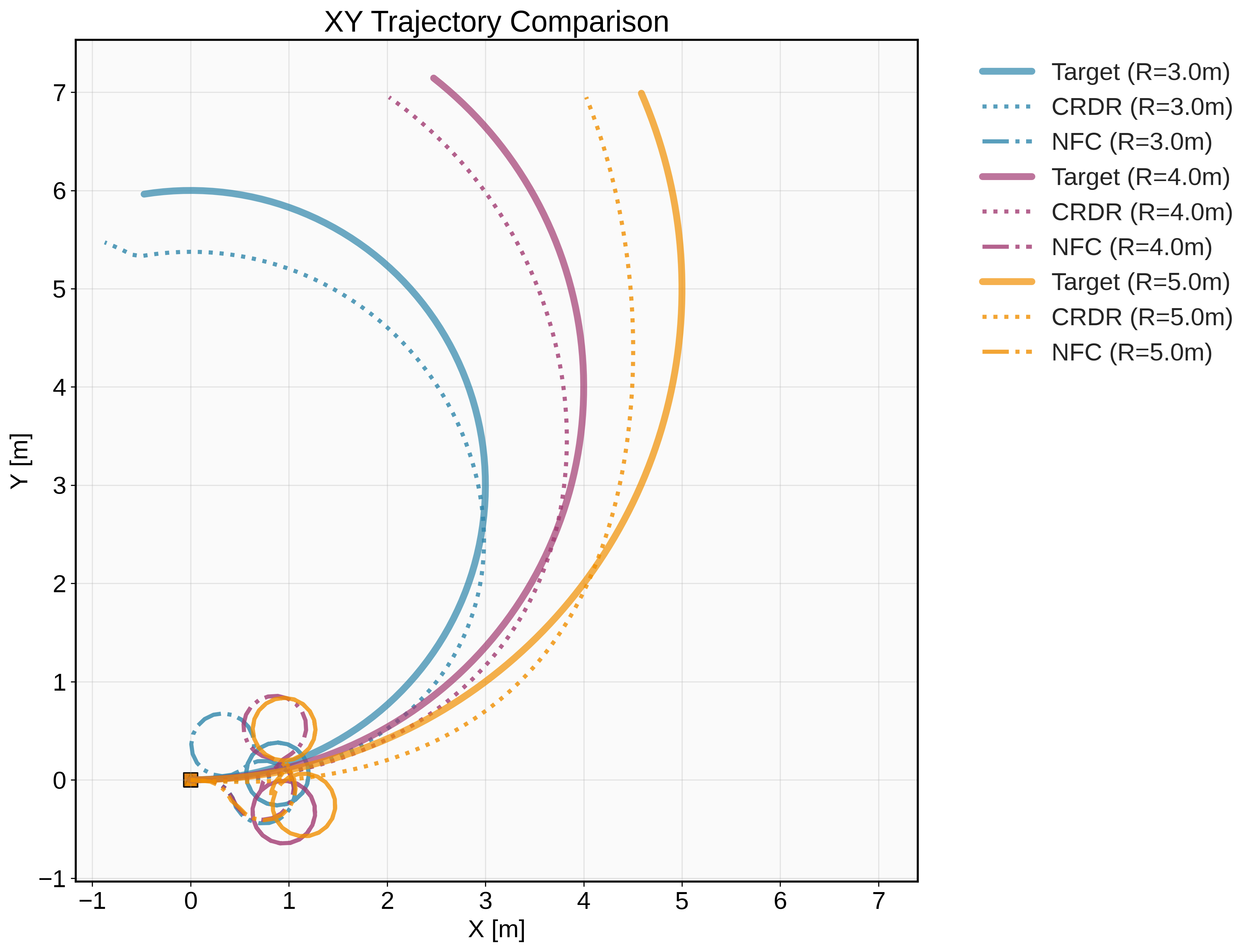}
    \includegraphics[width=1\linewidth]{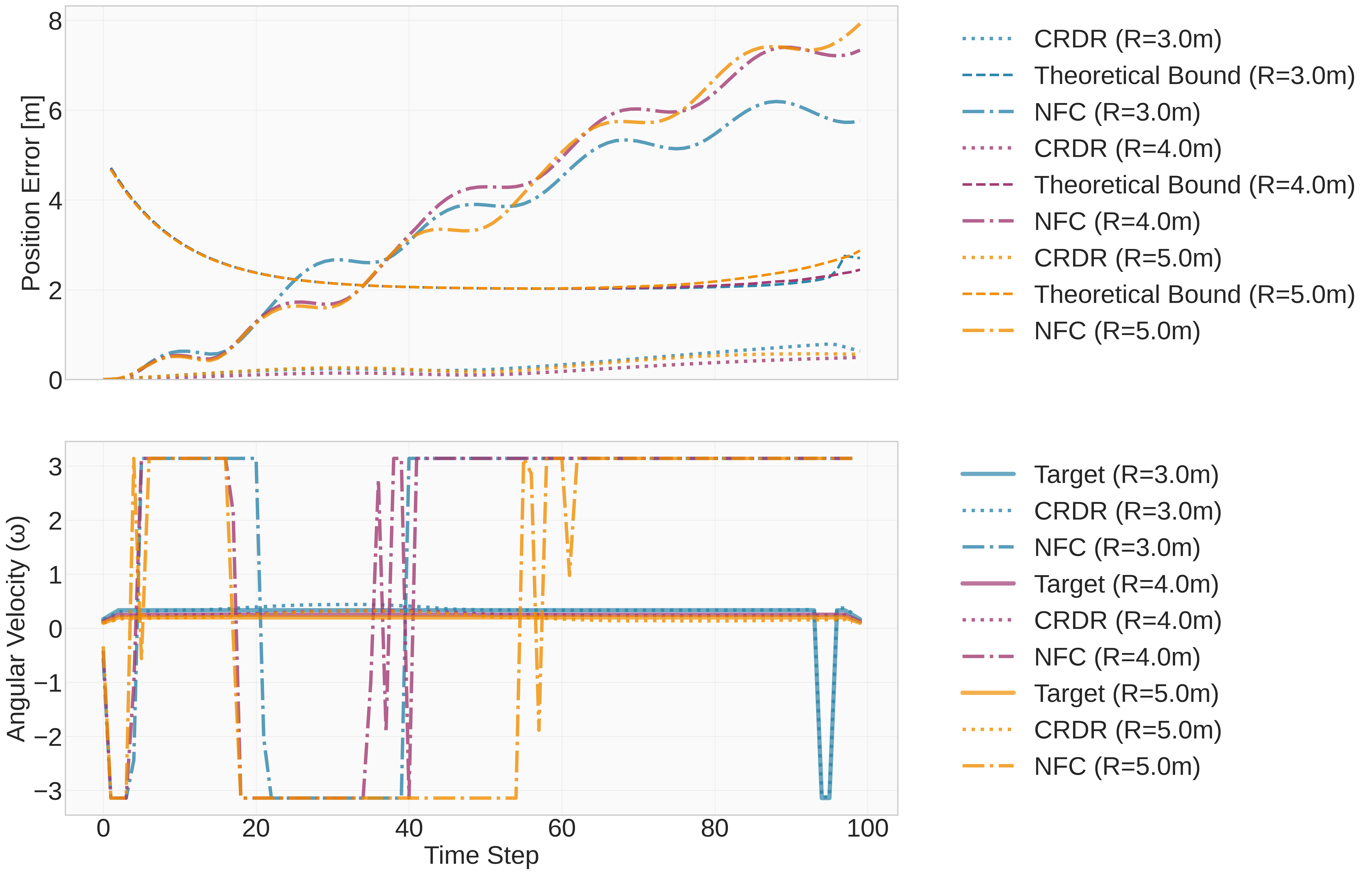}
    \caption{Controller performance comparison. (Top) XY trajectory tracking. (Bottom) Position error with theoretical bounds and the corresponding angular velocity control input ($\omega$) over time.}
    \label{fig:results}
    \vspace{-1em}
\end{figure}
\section{Experiments with a Flapping-Wing Robot}

This section presents the experimental validation of our proposed framework using a flapping-wing robot, hereafter referred to as the Flapper Drone~\cite{flapper_original} (\emph{Flapper Nimble+ from Flapper Drones}, \myhref{https://flapper-drones.com/}{https://flapper-drones.com/}). The experiments are designed to 1) validate the distribution-free error bounds derived using our method, and 2)~demonstrate the improved trajectory tracking performance compared to a built-in PID controller.

\subsection{Experimental Setup}
The Flapper Drone is characterized by highly nonlinear aerodynamic and inertial properties, making it a challenging and suitable platform for evaluating data-driven control methodologies. The drone is equipped with a manufacturer-provided PID controller for low-level attitude and position stabilization. Our approach introduces a high-level, Koopman-based control law that provides target position commands, $\bm{u}_k$, integrated into the underlying PID. We treat the onboard PID controller as part of the unknown nonlinear dynamics, $f$, aiming to model: $\bm{x}_{k+1} = f(\bm{x}_k, \text{PID}(\bm{u}_k))$, $\bm{z}_k = \phi(\bm{x}_k)$, $\bm{z}_{k+1} = A\bm{z}_k + B\bm{u}_k$, and $\hat{\bm{x}}_{k+1} = \psi(\bm{z}_{k+1})$. \added{Here, $\phi$ and $\psi$ are the lifting and decoding functions, respectively,} $\bm{x}_k$ is the drone's state, and our high-level control input \replaced{$\bm{u}_k = (x_{d,k}, y_{d,k}, z_{d,k})$}{$\bm{u}_k = (x_t, y_t, z_t)$} commands the target position. The primary objective is to accurately track circular trajectories with a radius of 0.5 m. \added{This radius ensures the drone remains within the reliable tracking volume of our motion capture system while allowing sufficient margin for the theoretical error bound.}

\subsection{Data Collection and Model Training}
To learn the system dynamics, we first collect a dataset of the drone's position and orientation. This is achieved by commanding the drone to track a predefined circular trajectory with a radius of 0.5 m using its default PID controller. From this dataset, we train a neural network as a surrogate model to predict the one-step-ahead state, $\bm{x}_{k+1}$, given the current state, $\bm{x}_k$, and a target position, $\bm{u}_k$. Subsequently, the Koopman operator for the system is learned from the data using our proposed methodology. \added{The architecture follows an encoder ($\phi$: $12\to1024\to16$), linear dynamics ($\bm{z}_{k+1}=A\bm{z}_k+B\bm{u}_k$, $A\in\mathbb{R}^{16\times16}$, $B\in\mathbb{R}^{16\times3}$), and decoder ($\psi$: $16\to1024\to12$) pipeline with ReLU and batch normalization. The model uses only the current state $\bm{x}_k$ (12-dim: 3D position and flattened rotation matrix) and input $\bm{u}_k$, without stacking past states, yielding sufficient prediction accuracy.}

\subsection{Validation of the Assured Tracking Bound}
To calculate the theoretical tracking error bound, we first generate a calibration dataset for conformal prediction. This is achieved by using the trained surrogate model to simulate CRDR.  CRDR is designed with parameters $c_v=100$, $\rho=1.0$ and $\gamma=0.9$. During this simulation, we sample the necessary data to estimate the quantiles for our error bound calculation. From the learned Koopman dynamics, the singular values of $\Theta$ were calculated as $\sqrt{\overline{m}/\underline{m}} = 1.00$, which suggests isotropic convergence in latent space. The calibration process for conformal prediction yields a forward prediction quantile of $q_{\text{fwd}} = 0.1365$ and a round-trip quantile of $q_{\text{rt}} = 0.4691$. These values are then used to compute the time-varying theoretical tracking error bound. \added{Finding stabilizing controller gains required significant tuning effort; we utilized the learned surrogate model to simulate and optimize these parameters offline, avoiding hardware risk.}

To validate this bound, we conduct flight experiments where the drone tracks circular trajectories in various planes. Figure \ref{fig:error_bound_validation} shows the real-time tracking error from these experiments. It indicates that the tracking error consistently remains below the computed theoretical bound for all trajectories, thereby implying the validity of our approach.

\begin{figure}[t]
    \centering
    \includegraphics[width=1\linewidth]{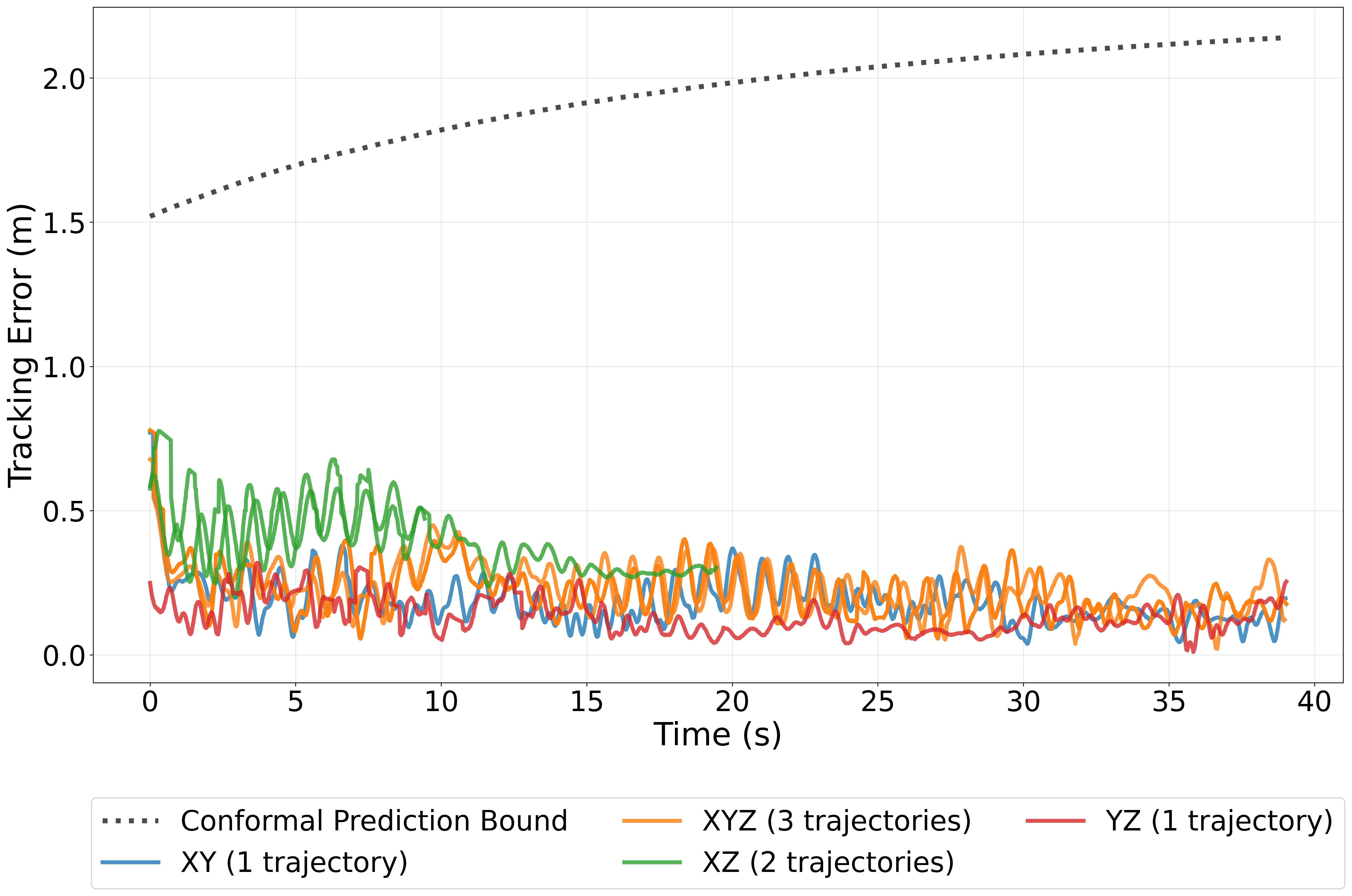}
    \caption{Tracking error of the Koopman-based controller for various trajectories. The measured error (solid lines) consistently stays below the theoretically computed bound (dotted line), validating our assured error bound estimation. \added{For a visual illustration of the spatial trajectories in these planes, see the YZ-plane example in Fig. 6.}}
    \label{fig:error_bound_validation}
    \vspace{-1em}
\end{figure}

\subsection{Improved Tracking Performance}
Finally, we compare tracking performance against the default PID controller on a YZ-plane circular trajectory (Figure~\ref{fig:tracking_comparison}). The PID controller exhibits significant oscillations, while our Koopman-based controller tracks the reference with markedly higher precision, demonstrating the effectiveness of our data-driven approach.

\begin{figure}[t]
    \centering
    \includegraphics[width=1\linewidth]{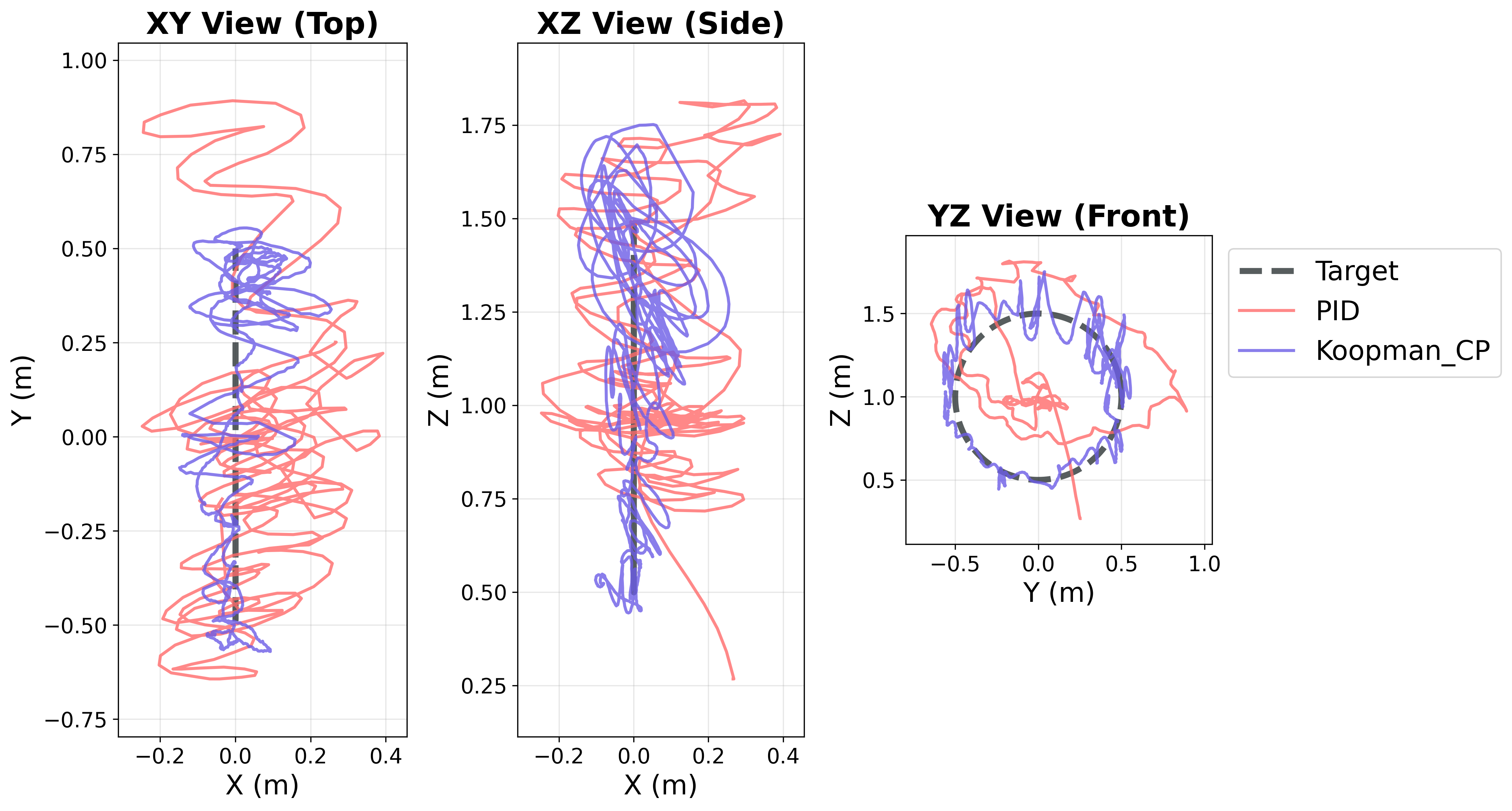}
    \caption{Comparison of tracking performance for a circular trajectory in the YZ plane. The proposed Koopman-based controller (Koopman\_CP, purple) demonstrates significantly reduced oscillations and superior tracking accuracy compared to the default PID controller (red).}
    \label{fig:tracking_comparison}
    \vspace{-1em}
\end{figure}
\section{Conclusion}
\label{sec_conclusion}
We present a framework that integrates conformal prediction with \emph{closed-loop} Koopman-based control of discrete-time nonlinear systems, with its theoretical connection to incremental stability in contraction theory. It provides distribution-free tracking error bounds by quantifying latent space prediction errors without assuming a specific error structure. Our analysis shows that the tracking error depends critically on not just latent modeling error, but also the feedback gain and the decoder's Lipschitz constant, offering a key control design guideline. The framework was validated in simulations with a Dubins car and hardware experiments with a flapping-wing robot, where measured tracking errors consistently stayed within our theoretical bounds.

Future work will aim to reduce the bounds' conservativeness and address exchangeability challenges when tracking diverse trajectories. \added{In particular, our safety guarantees are valid only within the calibration data distribution; out-of-distribution large disturbances or hardware degradation could invalidate them. Extending the evaluation to diverse trajectory families and benchmarking against other high-level planners is also an important future direction.}
\bibliographystyle{IEEEtran}
\bibliography{bib}

@article{Lusch2018,
  title={Deep learning for universal linear embeddings of nonlinear dynamics},
  author={Lusch, Bethany and Kutz, J Nathan and Brunton, Steven L},
  journal={Nature Communications},
  volume={9},
  number={1},
  pages={1--10},
  year={2018},
  publisher={Nature Publishing Group}
}

@article{Nuske2023,
  title={Finite-data error bounds for {Koopman}-based prediction and control},
  author={N{\"u}ske, Feliks and Peitz, Sebastian and Philipp, Friedrich and Schaller, Manuel and Worthmann, Karl},
  journal={Journal of Nonlinear Science},
  volume={33},
  number={14},
  year={2023},
  publisher={Springer}
}

@inproceedings{Mamakoukas2022ACC,
  title={Robust model predictive control with data-driven {Koopman} operators},
  author={Mamakoukas, George and Cairano, Stefano Di and Vinod, Ashish P},
  booktitle={American Control Conference},
  pages={1895--1902},
  year={2022},
  organization={IEEE}
}

@article{Zhang2022Automatica,
  title={Robust tube-based model predictive control with {Koopman} operators},
  author={Zhang, Xiaodong and Pan, Wei and Scattolini, Riccardo and Yu, Shengbo and Xu, Xinjiong},
  journal={Automatica},
  volume={137},
  pages={110114},
  year={2022},
  publisher={Elsevier}
}

@article{Wang2023IJMLC,
  title={Learning-based robust model predictive control with data-driven {Koopman} operators},
  author={Wang, Min and Lou, Xiaofeng and Cui, Bin},
  journal={International Journal of Machine Learning and Cybernetics},
  volume={14},
  pages={3295--3321},
  year={2023},
  publisher={Springer}
}

@inproceedings{Han2022ICLR,
  title={DeSKO: Stability-assured robust control with a deep stochastic {Koopman} operator},
  author={Han, Mengdi and Euler-Rolle, Julian and Katzschmann, Robert},
  booktitle={International Conference on Learning Representations},
  year={2022}
}

@misc{Patel2024Conformal,
  title={Conformal robust control of linear systems},
  author={Patel, Yash and Rayan, Sarah and Tewari, Ambuj},
  year={2024},
  note={arXiv:2405.16250},
  howpublished={arXiv preprint}
}

@article{Korda2018,
  title={Linear predictors for nonlinear dynamical systems: {Koopman} operator meets model predictive control},
  author={Korda, Milan and Mezić, Igor},
  journal={Automatica},
  volume={93},
  pages={149--160},
  year={2018},
  publisher={Elsevier},
  doi={10.1016/j.automatica.2018.03.046}
}

@inproceedings{Mamakoukas2022,
  title={Robust model predictive control with data-driven {Koopman} operators},
  author={Mamakoukas, George and Di Cairano, Stefano and Vinod, Ashish P.},
  booktitle={American Control Conference},
  pages={1895--1902},
  year={2022},
  organization={IEEE},
  doi={10.23919/ACC53348.2022.9867538}
}

@article{flapper_original,
author = {Matěj Karásek  and Florian T. Muijres  and Christophe De Wagter  and Bart D. W. Remes  and Guido C. H. E. de Croon },
title = {A tailless aerial robotic flapper reveals that flies use torque coupling in rapid banked turns},
journal = {Science},
volume = {361},
number = {6407},
pages = {1089-1094},
year = {2018},
doi = {10.1126/science.aat0350},
}

@ARTICLE{koopman_contraction,
  author={Yi, Bowen and Manchester, Ian R.},
  journal={IEEE Transactions on Automatic Control}, 
  title={On the Equivalence of Contraction and {Koopman} Approaches for Nonlinear Stability and Control}, 
  year={2024},
  volume={69},
  number={7},
  pages={4336-4351},
  doi={10.1109/TAC.2023.3319051}}

@inproceedings{takeishi2017learning,
  title={Learning {Koopman} invariant subspaces for dynamic mode decomposition},
  author={Takeishi, Naoya and Kawahara, Yoshinobu and Yairi, Takehisa},
  booktitle={Advances in Neural Information Processing Systems},
  volume={30},
  year={2017}
}

@article{wehmeyer2018time,
  title={Time-lagged autoencoders: Deep learning of slow collective variables for molecular kinetics},
  author={Wehmeyer, Christoph and No{\'e}, Frank},
  journal={The Journal of Chemical Physics},
  volume={148},
  number={24},
  year={2018},
  publisher={AIP Publishing}
}

@inproceedings{morton2018deep,
  title={Deep dynamical modeling and control of unsteady fluid flows},
  author={Morton, Jeremy and Jameson, Antony and Kochenderfer, Mykel J and Witherden, Freddie},
  booktitle={Advances in Neural Information Processing Systems},
  volume={31},
  year={2018}
}

@article{liang2025safenavigationdynamicenvironments,
      title={Safe Navigation in Dynamic Environments Using Data-Driven {Koopman} Operators and Conformal Prediction}, 
      author={Kaier Liang and Guang Yang and Mingyu Cai and Cristian-Ioan Vasile},
      journal = {arXiv preprint},
      year={2025},
      eprint={2504.00352},
      archivePrefix={arXiv},
    volume  = {arXiv:2504.00352},
      primaryClass={cs.RO}, 
}

@inproceedings{conformal1,
author = {Vovk, Volodya and Gammerman, Alexander and Saunders, Craig},
title = {\href{https://eprints.soton.ac.uk/258960/1/Random_ICML99.pdf}{Machine-learning applications of algorithmic randomness}},
year = {1999},
isbn = {1558606122},
booktitle = {International Conference on Machine Learning},
pages = {444–453},
numpages = {10},
}

@book{conformalbook,
  title     = {\href{https://link.springer.com/book/10.1007/b106715}{Algorithmic Learning in a Random World}},
  author    = {Vladimir Vovk and Alexander Gammerman and Glenn Shafer},
  year      = {2022},
  publisher = {Springer},
  isbn      = {978-3-031-06648-1},
  doi       = {10.1007/978-3-031-06649-8},
}

@article{2024_Lindemann_CP-control-survey,
  author={Lindemann, Lars and Zhao, Yiqi and Yu, Xinyi and Pappas, George J and Deshmukh, Jyotirmoy V},
  title   = {\href{https://arxiv.org/abs/2409.00536}{{Formal verification and control with conformal prediction}}},
  journal = {arXiv preprint},
  volume  = {arXiv:2409.00536},
  year    = {2024},
  eprint  = {2409.00536},
  archivePrefix = {arXiv},
  primaryClass  = {eess.SY}
}

@ARTICLE{2023_Lindemann_CP-Planning,
  author={Lindemann, Lars and Cleaveland, Matthew and Shim, Gihyun and Pappas, George J.},
  journal={IEEE Robotics and Automation Letters}, 
  title={Safe Planning in Dynamic Environments Using Conformal Prediction}, 
  year={2023},
  volume={8},
  number={8},
  pages={5116-5123},
  doi={10.1109/LRA.2023.3292071}
}

@INPROCEEDINGS{2024_Zhou_ACP-CBF-MPC,
  author={Zhou, Hao and Zhang, Yanze and Luo, Wenhao},
  booktitle={American Control Conference}, 
  title={Safety-Critical Control with Uncertainty Quantification using Adaptive Conformal Prediction}, 
  year={2024},
  volume={},
  number={},
  pages={574-580},
  doi={10.23919/ACC60939.2024.10644391}
}

@inproceedings{2023_Sun_CP-Diffusion,
 author = {Sun, Jiankai and Jiang, Yiqi and Qiu, Jianing and Nobel, Parth and Kochenderfer, Mykel J and Schwager, Mac},
  booktitle={Advances in Neural Information Processing Systems},
 pages = {80324--80337},
 title = {Conformal Prediction for Uncertainty-Aware Planning with Diffusion Dynamics Model},
 volume = {36},
 year = {2023}
}

@INPROCEEDINGS{2023_Chee_WPC-MPC,
  author={Chee, Kong Yao and Hsieh, M. Ani and Pappas, George J.},
  booktitle = {IEEE Conference on Decision and Control},
  title={Uncertainty Quantification for Learning-based MPC using Weighted Conformal Prediction}, 
  year={2023},
  volume={},
  number={},
  pages={342-349},
  doi={10.1109/CDC49753.2023.10383587}
}

@InProceedings{2023_Dixit_ACP-Planning,
  title = 	 {Adaptive Conformal Prediction for Motion Planning among Dynamic Agents},
  author =       {Dixit, Anushri and Lindemann, Lars and Wei, Skylar X and Cleaveland, Matthew and Pappas, George J. and Burdick, Joel W.},
  booktitle = 	 {Learning for Dynamics \& Control Conference},
  pages = 	 {300--314},
  year = 	 {2023},
}

@INPROCEEDINGS{2023_Yang_CP-Sensor,
  author={Yang, Shuo and Pappas, George J. and Mangharam, Rahul and Lindemann, Lars},
  booktitle = {IEEE Conference on Decision and Control},
  title={Safe Perception-Based Control Under Stochastic Sensor Uncertainty Using Conformal Prediction}, 
  year={2023},
  volume={},
  number={},
  pages={6072-6078},
  doi={10.1109/CDC49753.2023.10384075}}

@article{TingWeiCProbust,
  author       = {Ting-Wei Hsu and Hiroyasu Tsukamoto},
  title        = {Statistical Guarantees in Data-Driven Nonlinear Control: Conformal Robustness for Stability and Safety},
  journal      = {IEEE Control Systems Letters},
  year         = {2025},
  volume       = {9},
  pages        = {997–1002}
}

@article{CPLinearStochastic,
  author       = {Eleftherios E. Vlahakis and Lars Lindemann and Pantelis Sopasakis and Dimos V. Dimarogonas},
  title        = {Conformal Prediction for Distribution-free Optimal Control of Linear Stochastic Systems},
  journal      = {IEEE Control Systems Letters},
  year         = {2024},
  volume       = {8},
  pages        = {2835–2840}
}

@inproceedings{SWei_conformalContraction,
  author        = {Sihang Wei and Melkior Ornik and Hiroyasu Tsukamoto},
  title         = {Conformal Contraction for Robust Nonlinear Control with Distribution-Free Uncertainty Quantification},
  booktitle = {IEEE Conference on Decision and Control, to appear},
  year   = {2025},
  month  = Dec,
  url           = {https://arxiv.org/pdf/2507.13613}
}

@InProceedings{confcont3,
  title = 	 {Uncertainty quantification and robustification of model-based controllers using conformal prediction},
  author =       {Chee, Kong Yao and Silva, Thales C. and Hsieh, M. Ani and Pappas, George J.},
  booktitle = 	 {Learning for Dynamics \& Control Conference},
  pages = 	 {528--540},
  year = 	 {2024},
  volume = 	 {242},
  month = 	 Jul,
}
\end{document}